\documentclass[final,onefignum,onetabnum,onealgnum]{siamart220329}

\usepackage{amsmath,amsfonts,amsopn,amssymb}
\usepackage{graphicx}
\usepackage[caption=false]{subfig}
\usepackage{multirow,relsize,makecell}
\usepackage{url}
\ifpdf
  \DeclareGraphicsExtensions{.eps,.pdf,.png,.jpg}
\else
  \DeclareGraphicsExtensions{.eps}
\fi
\usepackage{enumitem}

\usepackage{algorithm,algorithmic}

\numberwithin{theorem}{section}
\newsiamremark{remark}{Remark}
\newsiamremark{assumption}{Assumption}
\newsiamremark{example}{Example}
\crefname{remark}{Remark}{Remarks}
\crefname{assumption}{Assumption}{Assumptions}
\crefname{example}{Example}{Examples}

\newcommand\DualFL{\texttt{DualFL}}

\headers{DualFL:~Duality-based federated learning}{Jongho Park and Jinchao Xu}

\title{DualFL:~A Duality-based Federated Learning Algorithm with Communication Acceleration in the General Convex Regime\thanks{Submitted to arXiv.\funding{This work was supported by the KAUST Baseline Research Fund.}}
}

\author{Jongho Park\thanks{Computer, Electrical and Mathematical Science and Engineering Division, King Abdullah University of Science and Technology~(KAUST), Thuwal 23955, Saudi Arabia
 (\email{jongho.park@kaust.edu.sa}, \email{xu@multigrid.org}).}
\and
Jinchao Xu\footnotemark[2] \thanks{Department of Mathematics, Pennsylvania State University, University Park, PA, 16802, USA}
 }

\ifpdf
\hypersetup{
  pdftitle={DualFL: A Duality-based Federated Learning Algorithm with Communication Acceleration in the General Convex Regime},
  pdfauthor={Jongho Park and Jinchao Xu}
}
\fi

\begin{document}

\maketitle

\begin{abstract}
We propose a new training algorithm, named \DualFL{}~(\textbf{Dual}ized \textbf{F}ederated \textbf{L}earning), for solving distributed optimization problems in federated learning.
\DualFL{} achieves communication acceleration for very general convex cost functions, thereby providing a solution to an open theoretical problem in federated learning concerning cost functions that may not be smooth nor strongly convex.
We provide a detailed analysis for the local iteration complexity of \DualFL{} to ensure the overall computational efficiency of \DualFL{}.
Furthermore, we introduce a completely new approach for the convergence analysis of federated learning based on a dual formulation.
This new technique enables concise and elegant analysis, which contrasts the complex calculations used in existing literature on convergence of federated learning algorithms.
\end{abstract}



\section{Introduction}
\label{Sec:Introduction}
This paper is devoted to a novel approach to design efficient training algorithms for \textit{federated learning}~\cite{MMRHA:2017}.
Unlike standard machine learning approaches, federated learning encourages each client to have its own dataset and to update a local correction of a global model maintained by an orchestrating server via the local dataset and a local training algorithm.
Recently, federated learning has been considered as an important research topic in the field of machine learning as data becomes increasingly decentralized and privacy of individual data is an utmost importance~\cite{Kairouz:2021,LSTS:2020}.

In federated learning, it is assumed that communication costs dominate~\cite{MMRHA:2017}.
Hence, training algorithms for federated learning should be designed toward a direction that the amount of communication among the clients is reduced.
For example, \texttt{FedAvg}~\cite{MMRHA:2017}, one of the most popular training algorithms for federated learning, improves its communication efficiency by adopting local training.
Namely, multiple local gradient descent steps instead of a single step are performed in each client before communication among the clients.
In recent years, various local training approaches have been considered to improve the communication efficiency of federated learning; e.g., operator splitting~\cite{PW:2020}, dual averaging~\cite{YZR:2021}, augmented Lagrangian~\cite{ZHDYL:2021}, Douglas--Rachfold splitting~\cite{TPPN:2021}, client-level momentum~\cite{XWWY:2021}, and sharpness-aware minimization~\cite{QLDLTL:2022}.
Meanwhile, under some similarity assumptions on local cost functions, gradient sliding techniques can be adopted to reduce the computational burden of each client; see~\cite{KBBGS:2022} and references therein.

An important observation made in~\cite{KKMRSS:2020} is that data heterogeneity in federated learning can cause client drift, which in turn affects the convergence of federated learning algorithms.
Indeed, it was observed in~\cite[Figure~3]{KMR:2020} that a large number of local gradient descent steps without shifting of the local gradient leads to solution nonconvergence.
To address this issue, several gradient shift techniques that can compensate for client drift have been considered: \texttt{Scaffold}~\cite{KKMRSS:2020}, \texttt{FedDyn}~\cite{AZMMWS:2021}, \texttt{S-Local-SVRG}~\cite{GHR:2021}, \texttt{FedLin}~\cite{MJPH:2021}, and \texttt{Scaffnew}~\cite{MMSR:2022}.
These techniques achieve linear convergence rates of the training algorithms through carefully designed gradient shift techniques.

Recently, it was investigated in a pioneering work~\cite{MMSR:2022} that communication acceleration can be achieved by a federated learning algorithm if we use a tailored gradient shift scheme and a probabilistic approach for communication frequency.
Specifically, it was shown that \texttt{Scaffnew}~\cite{MMSR:2022} achieves the optimal $\mathcal{O}(\sqrt{\kappa} \log (1/\epsilon))$-communication complexity of distributed convex optimization~\cite{AS:2015} for the smooth strongly convex regime, where $\kappa$ is the condition number of the problem and $\epsilon$ measures a target level of accuracy; see also~\cite{HH:2023} for a tighter analysis.
Since then, several federated learning algorithms with communication acceleration have been considered; to name a few, \texttt{ProxSkip-VR}~\cite{MYR:2022}, \texttt{APDA-Inexact}~\cite{SKR:2022}, and \texttt{RandProx}~\cite{CR:2023}.
One may refer to~\cite{MYR:2022} for a historical survey on the theoretical progress of federated learning algorithms.

In this paper, we continue growing the list of federated learning algorithms with communication acceleration by proposing a novel algorithm called \DualFL{}~(\textbf{Dual}ized \textbf{F}ederated \textbf{L}earning).
The key idea is to establish a certain duality~\cite{Rockafellar:2015} between a model federated learning problem and a composite optimization problem.
By the nature of composite optimization problems~\cite{Nesterov:2013}, the dual problem can be solved efficiently by a forward-backward splitting algorithm with the optimal convergence rate~\cite{BT:2009,CP:2016,RC:2022}.
By applying the predualization technique introduced in~\cite{LG:2019,LN:2017} to an optimal forward-backward splitting method for the dual problem, we obtain our proposed algorithm \DualFL{}.
While each individual technique used in this paper is not new, a combination of the techniques yields the following desirable results:
\begin{itemize}
    \item \DualFL{} achieves the optimal $O(\sqrt{\kappa} \log (1/\epsilon))$-communication complexity in the smooth strongly convex regime.
    \item \DualFL{} achieves communication acceleration even when the cost function is either nonsmooth or non-strongly convex.
    \item \DualFL{} can adopt any optimization algorithm as its local solver, making it adaptable to each client's local problem.
    \item Communication acceleration of \DualFL{} is guaranteed in a deterministic manner. That is, both the algorithm and its convergence analysis do not rely on stochastic arguments.
\end{itemize}
In particular, we would like to highlight that \DualFL{} is the first federated learning algorithm that achieves communication acceleration for cost functions that may not be smooth nor strongly convex.
This solves an open theoretical problem in federated learning concerning communication acceleration for general convex cost functions.
In addition, the duality technique used in the convergence analysis presented in this paper is completely new in the area of federated learning, and it provides concise and elegant analysis, distinguishing itself from prior existing works with intricate calculations.

The remainder of this paper is organized as follows.
In \cref{Sec:Problem}, we state a model federated learning problem.
We introduce the proposed \DualFL{} and its convergence properties in \cref{Sec:Main}.
In \cref{Sec:Non_strong}, we introduce a regularization technique for \DualFL{} to deal with non-strongly convex problems.
We establish connections to existing federated learning algorithms in \cref{Sec:Comparison}.
In \cref{Sec:Math}, we establish a duality relation between \DualFL{} and a forward-backward splitting algorithm applied to a certain dual formulation.
We present numerical results of \DualFL{} in \cref{Sec:Numerical}.
Finally, we discuss limitations of this paper in \cref{Sec:Conclusion}.

\section{Problem description}
\label{Sec:Problem}
In this section, we present a standard mathematical model for federated learning.
In federated learning, it is assumed that each client possesses its own dataset, and that a local cost function is defined with respect to the dataset of each client.
Hence, we consider the problem of minimizing the average of $N$ cost functions stored on $N$ clients~\cite{Kairouz:2021,KKMRSS:2020,MYR:2022}:
\begin{equation}
\label{FL}
\min_{\theta \in \Omega} \left\{ E (\theta) := \frac{1}{N} \sum_{j=1}^N f_j (\theta) \right\},
\end{equation}
where $\Omega$ is a parameter space and $f_j \colon \Omega \rightarrow \mathbb{R}$, $1 \leq j \leq N$, is a continuous and convex local cost function of the $i$th client.
The local cost function $f_j$ depends on the dataset of the $j$th client, but not on those of the other clients.
We further assume that the cost function $E$ is coercive, so that ~\eqref{FL} admits a solution $\theta^* \in \Omega$~\cite[Proposition~11.14]{BC:2011}.
We note that we do not need to make any similarity assumptions for $f_j$~(cf.~\cite[Assumptions~2 and~3]{QLDLTL:2022}).
Since problems of the form~\eqref{FL} arise in various applications in machine learning and statistics~\cite{SB:2014,Spokoiny:2012}, a number of algorithms have been developed to solve~\eqref{FL}, e.g., stochastic gradient methods~\cite{Bertsekas:2011,Bertsekas:2015,LZ:2018,ZX:2019}.

In what follows, an element of $\Omega^N$ is denoted by a bold symbol.
For $\boldsymbol{\theta} \in \Omega^N$ and $1 \leq j \leq N$, we denote the $j$th component of $\boldsymbol{\theta}$ by $\theta_j$, i.e., $\boldsymbol{\theta} = (\theta_j)_{j=1}^N$.
We use the notation $A \lesssim B$ to represent that $A \leq CB$ for some constant $C > 0$ independent of the number of iterations $n$.

We recall that, for a convex differentiable function $h \colon \Omega \rightarrow \mathbb{R}$, $h$ is said to be $\mu$-strongly convex for some $\mu >0$ when
\begin{equation*}
    h (\theta) \geq h (\phi) + \langle \nabla h (\phi), \theta - \phi \rangle + \frac{\mu}{2} \| \theta - \phi \|^2,
    \quad \theta, \phi \in \Omega.
\end{equation*}
In addition, $h$ is said to be $L$-smooth for some $L > 0$ when
\begin{equation*}
    h (\theta) \geq h (\phi) + \langle \nabla h (\phi), \theta - \phi \rangle + \frac{\mu}{2} \| \theta - \phi \|^2,
    \quad \theta, \phi \in \Omega.
\end{equation*}

\section{Main results}
\label{Sec:Main}
In this section, we present the main results of this paper: the proposed algorithm, called \DualFL{}, and its convergence theorems.
We now present \DualFL{} in \cref{Alg:DualFL} as follows.

\begin{algorithm}
\caption{\DualFL{}: \textbf{Dual}ized \textbf{F}ederated \textbf{L}earning}
\label{Alg:DualFL}
\begin{algorithmic}[]
\STATE Given $\rho \geq 0$ and $\nu > 0$,
\STATE set $\theta^{(0)} = \theta_j^{(0)} = 0 \in \Omega$~($1 \leq j \leq N$), $\boldsymbol{\zeta}^{(0)} = \boldsymbol{\zeta}^{(-1)} = \textbf{0} \in \Omega^N$, and $t_0 = 1$.
\FOR{$n= 0,1,2, \dots$}
\FOR{\textbf{each client} ($1 \leq j \leq N$) \textbf{in parallel}}
\STATE Find $\theta_j^{(n+1)}$ by $M_n$ iterations of some local training algorithm for the following problem:
\begin{equation}
\label{local_primal}
\theta_j^{(n+1)} \approx \operatornamewithlimits{\arg\min}_{\theta_j \in \Omega} \left\{ E^{n,j} (\theta_j) := f_j (\theta_j) - \nu \langle \zeta_j^{(n)}, \theta_j \rangle \right\}
\end{equation}
\ENDFOR
\STATE
\begin{equation}
    \label{theta}
    \theta^{(n+1)} = \frac{1}{N} \sum_{j=1}^N  \theta_j^{(n+1)}
\end{equation}
\FOR{\textbf{each client} ($1 \leq j \leq N$) \textbf{in parallel}}
\STATE
\begin{equation}
    \label{zeta}
    \zeta_j^{(n+1)} = (1 + \beta_n) \left( \zeta_j^{(n)} + \theta^{(n+1)} - \theta_j^{(n+1)} \right) - \beta_n \left( \zeta_j^{(n-1)} + \theta^{(n)} - \theta_j^{(n)} \right),
\end{equation}
where $\beta_n$ is given by
\begin{equation}
\label{FISTA}
t_{n+1} = \frac{1 - \rho t_n^2 + \sqrt{(1 - \rho t_n^2)^2 + 4t_n^2}}{2}, \quad
\beta_n = \frac{t_n - 1}{t_{n+1}} \frac{1 - t_{n+1} \rho}{1 - \rho}.
\end{equation}
\ENDFOR
\ENDFOR
\end{algorithmic}
\end{algorithm}

\DualFL{} updates the server parameter from $\theta^{(n)}$ to $\theta^{(n+1)}$ by the following steps.
First, each client computes its local solution $\theta_j^{(n+1)}$ by solving the local problem~\eqref{local_primal} with several iterations of some local training algorithm.
Note that, similar to \texttt{Scaffold}~\cite{KKMRSS:2020}, \texttt{FedLin}~\cite{MJPH:2021}, and \texttt{Scaffnew}~\cite{MMSR:2022}, the local problem~\eqref{local_primal} is defined in terms of the local control variate $\zeta_j^{(n)}$.
Then the server aggregates all the local solutions $\theta_j^{(n+1)}$ by averaging them to obtain a new server parameter $\theta^{(n+1)}$.
After obtaining the new server parameter $\theta^{(n+1)}$, it is transferred to each client, and the local control variate is updated using~\eqref{zeta}.
The overrelaxation parameter $\beta_n$ in~\eqref{zeta} can be obtained by a simple recursive formula~\eqref{FISTA}, which relies on the hyperparameter $\rho$.
We note that a similar overrelaxation scheme was employed in \texttt{APDA-Inexact}~\cite{SKR:2022}.

One feature of the proposed \DualFL{} is its flexibility in choosing local solvers for the local problem~\eqref{local_primal}.
More precisely, the method allows for the adoption of any local solvers, making it adaptable to each local problem in a client.
The same advantage was reported in several existing works such as~\cite{AZMMWS:2021,TPPN:2021,ZHDYL:2021}.
Another notable feature of \DualFL{} is its fully deterministic nature, in contrast to some existing federated learning algorithms that rely on randomness to achieve communication acceleration~\cite{CR:2023,MYR:2022,MMSR:2022}.
Specifically, \DualFL{} does not rely on uncertainty to ensure communication acceleration, which enhances its reliability.
Very recently, several federated learning algorithms that share the same advantage have been proposed; see, e.g.,~\cite{SKR:2022}.

\subsection{Inexact local solvers}
In \DualFL{}, local problems of the form~\eqref{local_primal} are typically solved inexactly using iterative algorithms.
The resulting local solutions may deviate from the exact minimizers, and this discrepancy can affect the convergence behavior.
Here, we present a certain inexactness assumption for local solvers that does not deteriorate the convergence properties of \DualFL{}.

For a function $f \colon X \rightarrow \overline{\mathbb{R}}$ defined on a Euclidean space $X$, let $f^* \colon X \rightarrow \overline{\mathbb{R}}$ denote the Legendre--Fenchel conjugate of $f$, i.e.,
\begin{equation*}
    f^* (p) = \sup_{x \in X} \left\{ \left< p, x \right> - f(x) \right\}, \quad p \in X.
\end{equation*}
The following proposition is readily deduced by the Fenchel--Rockafellar duality~(see \cref{Prop:FR}).

\begin{proposition}
\label{Prop:local_dual}
Suppose that each $f_j$, $1 \leq j \leq N$, in~\eqref{FL} is $\mu$-strongly convex for some $\mu > 0$.
For a positive constant $\nu \in (0, \mu]$, if $\theta_j \in \Omega$ solves~\eqref{local_primal}, then $\xi_j = \nu (\zeta_j^{(n)} - \theta_j) \in \Omega$ solves
\begin{equation}
    \label{local_dual}
    \min_{\xi_j \in \Omega} \left\{ E_{\mathrm{d}}^{n,j} (\xi_j)
    := g_j^* (\xi_j) + \frac{1}{2\nu} \| \xi_j - \nu \zeta_j^{(n)} \|^2 \right\},
\end{equation}
where $g_j (\theta) = f_j (\theta) - \frac{\nu}{2} \| \theta \|^2$.
Moreover, we have
\begin{equation*}
    E^{n,j} (\theta_j) + E_{\mathrm{d}}^{n,j} (\xi_j) = 0.
\end{equation*}
\end{proposition}

Thanks to \cref{Prop:local_dual}, $\theta_j \in \Omega$ is a solution of~\cref{local_primal} if and only if the primal-dual gap $\Gamma^{n,j} (\theta_j)$ defined by
\begin{equation}
    \label{Gamma}
    \Gamma^{n,j} (\theta_j) = 
    E^{n,j} (\theta_j) + E_{\mathrm{d}}^{n,j} ( \nu (\zeta_j^{(n)} - \theta_j) )
\end{equation}
vanishes~\cite{BRR:2018}.
The primal-dual gap $\Gamma^{n,j} (\theta_j)$ can play a role of an implementable inexactness criterion since it is observable by simple arithmetic operations~(see Section~2.1 of~\cite{BTB:2022}).

If the local problem~\eqref{local_primal} is solved by a convergent iterative algorithm such as gradient descent methods, then the primal-dual gap $\Gamma^{n,j} (\theta_j^{(n+1)})$ can be arbitrarily small with a sufficiently large number of inner iterations.

\subsection{Convergence theorems}
The following theorem states that \DualFL{} is provably convergent in the nonsmooth strongly convex regime if each local problem is solved so accurately that the primal-dual gap becomes less than a certain value.
Moreover, \DualFL{} achieves communication acceleration in the sense that the squared solution error $\| \theta^{(n)} - \theta^* \|^2$ at the $n$th communication round is bounded by $\mathcal{O}(1/n^2)$, which is derived by momentum acceleration; see \cref{Sec:Math}.
As we are aware, \DualFL{} is the first federated learning algorithm with communication acceleration that is convergent even if the cost function is nonsmooth.
A proof of \cref{Thm:nonsmooth} will be provided in \cref{Sec:Math}.

\begin{theorem}
\label{Thm:nonsmooth}
Suppose that each $f_j$, $1 \leq j \leq N$, in~\eqref{FL} is $\mu$-strongly convex for some $\mu > 0$.
In addition, suppose that the number of local iterations for the $j$th client at the $n$th epoch of \DualFL{} is large enough to satisfy
\begin{equation}
    \label{nonsmooth_inexact}
    \Gamma^{n,j} (\theta_j^{(n+1)}) \leq \frac{1}{N \nu (n+1)^{4+\gamma}}
\end{equation}
for some $\gamma > 0$~($1 \leq j \leq N$, $n \geq 0$).
If we choose the hyperparameters $\rho$ and $\nu$ in \DualFL{} such that $\rho  = 0$ and $\nu \in (0, \mu]$, then the sequence $\{ \theta^{(n)} \}$ generated by \DualFL{} converges to the solution $\theta^*$ of~\eqref{FL}.
Moreover, for $n \geq 0$, we have
\begin{equation*}
    \| \theta^{(n)} - \theta^* \|^2 \lesssim \frac{1}{n^2}.
\end{equation*}
\end{theorem}

If we additionally assume that each $f_j$ in~\eqref{FL} is $L$-smooth for some $L > 0$, then we are able to obtain an improved convergence rate of \DualFL{}.
Under the $\mu$-strong convexity and $L$-smoothness assumptions on $f_j$, we define the condition number $\kappa$ of the problem~\eqref{FL} as $\kappa = L/\mu$.
If we choose the hyperparameters $\rho$ and $\nu$ appropriately, then \DualFL{} becomes linearly convergent with the rate $1 - 1/\sqrt{\kappa}$.
Consequently, \DualFL{} achieves the optimal $\mathcal{O}(\sqrt{\kappa} \log (1/\epsilon))$-communication efficiency~\cite{AS:2015}.
This observation is summarized in \cref{Thm:smooth}; see \cref{Sec:Math} for a proof.

\begin{theorem}
\label{Thm:smooth}
Suppose that each $f_j$, $1 \leq j \leq N$, in~\eqref{FL} is $\mu$-strongly convex and $L$-smooth for some $\mu, L > 0$.
In addition, suppose that the number of local iterations for the $j$th client at the $n$th epoch of \DualFL{} is large enough to satisfy
\begin{equation}
    \label{smooth_inexact}
    \Gamma^{n,j} (\theta_j^{(n+1)}) \leq \frac{1}{N} \left( \frac{1 - \sqrt{\rho}}{1 + \gamma} \right)^n
\end{equation}
for some $\gamma > 0$~($1 \leq j \leq N$, $n \geq 0$).
If we choose the hyperparameters $\rho$ and $\nu$ in \DualFL{} such that $\rho  \leq [0, \nu / L]$ and $\nu \leq (0, \mu]$, then the sequence $\{ \theta^{(n)} \}$ generated by \DualFL{} converges to the solution $\theta^*$ of~\eqref{FL}.
Moreover, for $n \geq 0$, we have
\begin{equation*}
    E(\theta^{(n)}) - E(\theta^*) \lesssim \| \theta^{(n)} - \theta^* \|^2 \lesssim \left( 1 - \sqrt{\rho} \right)^n.
\end{equation*}
In particular, if we set $\rho = \kappa^{-1}$ and $\nu = \mu$ in \DualFL{}, then we have
\begin{equation*}
    E(\theta^{(n)}) - E(\theta^*) \lesssim \| \theta^{(n)} - \theta^* \|^2 \lesssim \left( 1 - \frac{1}{\sqrt{\kappa}} \right)^n,
\end{equation*}
where $\kappa = L/\mu$.
Namely, \DualFL{} achieves the optimal $\mathcal{O} ( \sqrt{\kappa} \log (1/\epsilon))$-communication complexity of distributed convex optimization in the smooth strongly convex regime.
\end{theorem}

\cref{Thm:smooth} implies that \DualFL{} is linearly convergent with an acceptable rate $1 - \sqrt{\rho}$ even if the hyperparameters were not chosen optimally.
That is, the performance \DualFL{} is robust with respect to a choice of the hyperparameters.

\subsection{Local iteration complexity}
We analyze the local iteration complexity of \DualFL{} under the conditions of \cref{Thm:nonsmooth,Thm:smooth}.
We recall that \DualFL{} is compatible with any optimization algorithm as its local solver.
Hence, we may assume that we use an optimal first-order optimization algorithm in the sense of Nemirovskii and Nesterov~\cite{NN:1985,Park:2022}.
That is, optimization algorithms of iteration complexity $\mathcal{O}(1/\epsilon)$ and $\mathcal{O}(\sqrt{\kappa}\log (1/\epsilon) )$ are considered in the cases corresponding to \cref{Thm:nonsmooth,Thm:smooth}.
Based on this setting, we have the following results regarding the local iteration complexity of \DualFL{}.
Both theorems can be derived straightforwardly by substituting $\epsilon$ in the iteration complexity of local solvers with the threshold values given in \cref{Thm:nonsmooth,Thm:smooth}.
Note that the number of outer iterations of \DualFL{} to meet the target accuracy $\epsilon_{\mathrm{out}} > 0$ is $\mathcal{O} ( 1 / \sqrt{\epsilon_{\mathrm{out}}})$ and $\mathcal{O} ( (1/\sqrt{\rho}) \log (1/ \epsilon_{\mathrm{out}}) )$ in the cases of \cref{Thm:nonsmooth,Thm:smooth}, respectively.

\begin{theorem}
    \label{Thm:local_nonsmooth}
    Suppose that the assumptions given in \cref{Thm:nonsmooth} hold.
    If the local problem~\eqref{local_primal} is solved by an optimal first-order algorithm of iteration complexity $\mathcal{O}(1/\epsilon)$, then the number of inner iterations $M_n$ at the $n$th epoch of \DualFL{} satisfies
    \begin{equation*}
        M_n = \mathcal{O} \left( N (n+1)^{4+\gamma} \right)
        = \mathcal{O} \left( \frac{N}{\epsilon_{\mathrm{out}}^{2 + \frac{\gamma}{2}}} \right),
    \end{equation*}
    where $\epsilon_{\mathrm{out}} > 0$ is the target accuracy of the outer iterations of \DualFL{}.
\end{theorem}

\begin{theorem}
    \label{Thm:local_smooth}
    Suppose that the assumptions given in \cref{Thm:smooth} hold.
    If the local problem~\eqref{local_primal} is solved by an optimal first-order algorithm of iteration complexity $\mathcal{O}(\sqrt{\kappa}\log (1/\epsilon))$, then the number of inner iterations $M_n$ at the $n$th epoch of \DualFL{} satisfies
    \begin{equation*}
        M_n = \mathcal{O} \left( \sqrt{\kappa} \left( \log (N ( 1 + \gamma)^n ) + n \sqrt{\rho} \right) \right)
        = \mathcal{O} \left( \sqrt{\kappa} \log \frac{N (1 + \gamma)^n}{\epsilon_{\mathrm{out}}}\right),
    \end{equation*}
    where $\epsilon_{\mathrm{out}} > 0$ is the target accuracy of the outer iterations of \DualFL{}.
    In particular, if we set $\gamma \rightarrow 0^+$, then we have
    \begin{equation*}
        M_n = \mathcal{O} \left(\sqrt{\kappa} \log \frac{N}{\epsilon_{\mathrm{out}}} \right).
    \end{equation*}
\end{theorem}

Similar to other state-of-the-art federated learning algorithms~\cite{CR:2023,GMR:2022,MMSR:2022}, the local iteration complexity of \DualFL{} scales with $\sqrt{\kappa}$.
This implies that \DualFL{} is computationally efficient, not only in terms of communication complexity but also in terms of total complexity.

\section{Extension to non-strongly convex problems}
\label{Sec:Non_strong}
The convergence properties of the proposed \DualFL{} presented in \cref{Sec:Main} rely on the strong convexity of the cost function $E$ in~\eqref{FL}.
Although this assumption has been considered as a standard one in many existing works on federated learning algorithms~\cite{AZMMWS:2021,KKMRSS:2020,MMSR:2022,SKR:2022}, it may not hold in practical settings and is often unrealistic.
In this section, we deal with how to apply \DualFL{} to non-strongly convex problems, i.e., when $E$ is not assumed to be strongly convex.
Throughout this section, we assume that each $f_j$, $1 \leq j \leq N$, in the model problem~\eqref{FL} is not strongly convex.
In this case,~\eqref{FL} admits nonunique solutions in general.
For a positive real number $\alpha > 0$, we consider the following $\ell^2$-regularization~\cite{Nesterov:2012} of~\eqref{FL}:
\begin{equation}
    \label{FL_reg}
    \min_{\theta \in \Omega} \left\{ E^{\alpha}(\theta) := \frac{1}{N} \sum_{j=1}^N f_j^{\alpha} (\theta) \right\},
    \quad
    f_j^{\alpha} (\theta) = f_j (\theta) + \frac{\alpha}{2} \| \theta \|^2.
\end{equation}
Then each $f_j^{\alpha}$ in~\eqref{FL_reg} is $\alpha$-strongly convex.
Hence, \DualFL{} applied to~\eqref{FL_reg} satisfy the convergence properties stated in \cref{Thm:nonsmooth,Thm:smooth}.
In particular, the sequence $\{ \theta^{(n)} \}$ generated by \DualFL{} applied to~\eqref{FL_reg} converges to the unique solution $\theta^{\alpha} \in \Omega$ of~\eqref{FL_reg}.
Invoking the epigraphical convergence theory from~\cite{RW:2009}, we establish \cref{Thm:non_strong}, which means that for sufficiently small $\alpha$ and large $n$, $\theta^{(n)}$ is a good approximation of a solution $\theta^*$ of~\eqref{FL}.

\begin{theorem}
    \label{Thm:non_strong}
    In \DualFL{} applied to the regularized problem~\eqref{FL_reg}, suppose that the local problems are solved with sufficient accuracy so that~\eqref{nonsmooth_inexact} holds.
    If we choose $\rho = 0$ and $\nu = \alpha$ in \DualFL{}, then the sequence $\{ \theta^{(n)} \}$ generated by \DualFL{} applied to~\eqref{FL_reg} satisfies
    \begin{equation*}
        E(\theta^{(n)}) - E(\theta^*) \rightarrow 0 \quad \text{ as }\hspace{0.1cm} n \rightarrow \infty \hspace{0.1cm}\text{ and }\hspace{0.1cm} \alpha \rightarrow 0^+.
    \end{equation*}
\end{theorem}
\begin{proof}
It is clear that $E^{\alpha}$ decreases to $E$ as $\alpha \rightarrow 0^+$.
Hence, by~\cite[Proposition~7.4]{RW:2009}, $E^{\alpha}$ epi-converges to $E$.
Since $E$ is coercive, we conclude by~\cite[Theorem~7.33]{RW:2009} that
\begin{equation}
    \label{Thm:non_strong_1}
    E(\theta^{\alpha}) \rightarrow E(\theta^*) \quad \text{ as } \alpha \rightarrow 0^+.
\end{equation}
On the other hand, \cref{Thm:nonsmooth} implies that $\theta^{(n)} \rightarrow \theta^{\alpha}$ as $n \rightarrow \infty$.
As $E$ is continuous, we have
\begin{equation}
    \label{Thm:non_strong_2}
    E(\theta^{(n)}) \rightarrow E(\theta^{\alpha}) \quad \text{ as } n \rightarrow \infty.
\end{equation}
Combining~\eqref{Thm:non_strong_1} and~\eqref{Thm:non_strong_2} yields
\begin{equation*}
    E(\theta^{(n)}) - E(\theta^*) \rightarrow 0 \quad \text{ as }\hspace{0.1cm} n \rightarrow \infty \hspace{0.1cm}\text{ and }\hspace{0.1cm} \alpha \rightarrow 0^+,
\end{equation*}
which is our desired result.
\end{proof}

In the proof of \cref{Thm:non_strong}, we used the fact that $E(\theta^{\alpha}) \rightarrow E(\theta^*)$ as $\alpha \rightarrow 0^+$~\cite[Theorem~7.33]{RW:2009}.
Hence, by the coercivity of $E$, for any $\alpha_0 > 0$, we have $R_0 > 0$ such that
\begin{equation}
\label{R0}
\left\{ \theta^{\alpha} : \alpha \in (0, \alpha_0] \right\} \subset
\left\{ \theta : \| \theta \| \leq R_0 \right\}.
\end{equation}
If we assume that each $f_j$ in~\eqref{FL} is $L$-smooth, we can show that \DualFL{} achieves communication acceleration in the sense that the number of communication rounds to make the gradient error $\| \nabla E(\theta^{(n)}) \|$ smaller than $\epsilon$ is $\mathcal{O}((1/\sqrt{\epsilon}) \log (1/\epsilon))$, which agrees with the optimal estimate for first-order methods up to a logarithmic factor~\cite{KF:2021}.

\begin{theorem}
    \label{Thm:non_strong_rate}
    Suppose that each $f_j$, $1 \leq j \leq N$, in~\eqref{FL} is $L$-smooth for some $L > 0$.
    In addition, in \DualFL{} applied to the regularized problem~\eqref{FL_reg}, suppose that the local problems are solved with sufficiently accuracy so that~\eqref{smooth_inexact} holds.
    If we choose $\rho = \alpha/(L+\alpha)$ and $\nu = \alpha$ in \DualFL{}, then, for $n \geq 0$, we have
    \begin{equation}
        \label{Thm:non_strong_rate_1}
        \| \nabla E (\theta^{(n)}) \| \lesssim \left( 1 - \sqrt{\frac{\alpha}{L + \alpha}} \right)^{\frac{n}{2}} + \alpha \| \theta^{\alpha} \|.
    \end{equation}
    Moreover, if we choose $\alpha = \epsilon/(2 R_0) $ for some $\epsilon \in (0, 2R_0 \alpha_0]$, where $\alpha_0$ and $R_0$ were given in~\eqref{R0}, then the number of communication rounds $M_{\mathrm{comm}}$ to achieve  $\| \nabla E (\theta^{(n)}) \| \leq \epsilon$ satisfies
    \begin{equation}
        \label{Thm:non_strong_rate_2}
        M_{\mathrm{comm}} \leq \left( 1 + 2 \sqrt{1 + \frac{2 L R_0}{\epsilon}} \right) \left( \log \frac{1}{\epsilon} + \mathrm{constant} \right)
        = \mathcal{O} \left( \frac{1}{\sqrt{\epsilon}} \log \frac{1}{\epsilon} \right).
    \end{equation}
\end{theorem}
\begin{proof}
    Since $\theta^{\alpha}$ minimizes $E^{\alpha}$, we get
    \begin{equation}
        \label{Thm:non_strong_rate_3}
        \nabla E^{\alpha} (\theta^{\alpha}) = \nabla E (\theta^{\alpha}) + \alpha \theta^{\alpha} = 0.
    \end{equation}
    By the triangle inequality, $L$-smoothness,~\eqref{Thm:non_strong_rate_3}, and \cref{Thm:smooth}, we obtain
    \begin{equation*}
    \begin{split}
        \| \nabla E (\theta^{(n)}) \|
        &\leq \| \nabla E(\theta^{(n)}) - \nabla E(\theta^{\alpha}) \| + \| \nabla E (\theta^{\alpha}) \| \\
        &\leq L \| \theta^{(n)} - \theta^{\alpha} \| + \alpha \| \theta^{\alpha} \| \\
        &\lesssim \left( 1 - \sqrt{\frac{\alpha}{L + \alpha}} \right)^{\frac{n}{2}} + \alpha \| \theta^{\alpha} \|,
    \end{split}
    \end{equation*}
    which proves~\eqref{Thm:non_strong_rate_1}.

    Next, we proceed similarly as in~\cite[Theorem~3.3]{Park:2022}.
    Let $\epsilon \in (0, 2R_0 \alpha_0]$ and $\alpha = \epsilon / (2 R_0)$, so that we have $\alpha \leq \alpha_0$ and $\| \theta^{\alpha} \| \leq R_0$ by~\eqref{R0}.
    Then we obtain
    \begin{equation*}
        \| \nabla E(\theta^{(n)}) \|
        \leq C \left(1 - \sqrt{\frac{\epsilon}{\epsilon + 2L R_0}} \right)^{\frac{n}{2}} + \frac{\epsilon}{2}
        \leq C \left(1 + \sqrt{\frac{\epsilon}{\epsilon + 2L R_0}} \right)^{-\frac{n}{2}} + \frac{\epsilon}{2},
    \end{equation*}
    where $C$ is a positive constant independent of $n$.
    Consequently, $M_{\mathrm{comm}}$ is determined by the following equation:
    \begin{equation*}
        C \left(1 + \sqrt{\frac{\epsilon}{\epsilon + 2L R_0}} \right)^{-\frac{M_{\mathrm{comm}}}{2}} = \frac{\epsilon}{2}.
    \end{equation*}
    It follows that
    \begin{equation*}
        M_{\mathrm{comm}} = \frac{2 \log \frac{2C}{\epsilon}}{\log \left(1 + \sqrt{\frac{\epsilon}{\epsilon + 2L R_0}} \right)}
        \leq \left( 1 + 2 \sqrt{1 + \frac{2 L R_0}{\epsilon}} \right) \left( \log \frac{1}{\epsilon} + \log 2C \right),
    \end{equation*}
    where we used an elementary inequality~\cite[Equation~(3.5)]{Park:2022}
    \begin{equation*}
        \log \left( 1 + \frac{1}{t} \right) \geq \frac{2}{2t + 1}, \quad t > 0.
    \end{equation*}
    This proves~\eqref{Thm:non_strong_rate_2}.
\end{proof}

\section{Comparison with existing algorithms and convergence theory}
\label{Sec:Comparison}
\begin{table}
    \centering
    \caption{Comparison between \DualFL{} and other fifth-generation federated learning algorithms that achieve acceleration of communication complexity. The $\tilde{\mathcal{O}}$-notation neglects logarithmic factors.}
    \resizebox{\textwidth}{!}{ 
    \begin{tabular}{m{0.18\textwidth}cccccc}
    \multirow{3}{*}{Algorithm} & \multicolumn{3}{c}{Comm. acceleration} & \multicolumn{2}{c}{Local iter.\ complexity} & \multirow{3}{*}{\begin{tabular}{c} Deterministic \\ / Stochastic\end{tabular}}\\
    \cline{2-6}
    &  smooth & nonsmooth & smooth & smooth & nonsmooth & \\
    & \multicolumn{2}{c}{strongly convex} & non-strongly convex & \multicolumn{2}{c}{strongly convex} & \\
    \hline\hline
    \texttt{Scaffnew}~\cite{MMSR:2022}
    & Yes & N/A & N/A & $\tilde{\mathcal{O}} (\sqrt{\kappa})$ & N/A & Stochastic \\
    \hline
    \texttt{APDA-Inexact} \cite{SKR:2022}
    & Yes & N/A & N/A & better & N/A & Deterministic \\
    \hline
    \texttt{5GCS}~\cite{GMR:2022}
    & Yes & N/A & N/A & $\tilde{\mathcal{O}} (\sqrt{\kappa})$ & N/A & Deterministic\\
    \hline
    \texttt{RandProx}~\cite{CR:2023}
    & Yes & N/A & No & $\tilde{\mathcal{O}} (\sqrt{\kappa})$ & N/A & Stochastic \\
    \hline
    \DualFL{}
    & Yes & \textbf{Yes} & \textbf{Yes} & $\tilde{\mathcal{O}} (\sqrt{\kappa})$ & $\tilde{\mathcal{O}}(1/\epsilon^2)$ & Deterministic \\
    \end{tabular}
    }
    \label{Table:existing}
\end{table}

In this section, we discuss connections to existing federated learning algorithms.
Based on the classification established in~\cite{MYR:2022}, \DualFL{} can be classified as a fifth-generation federated learning algorithm, which achieves communication acceleration.
In the smooth strongly convex regime, \DualFL{} achieves the $\mathcal{O} (\sqrt{\kappa} \log (1/\epsilon))$-communication complexity, which is comparable to other existing algorithms in the same generation such as \texttt{Scaffnew}~\cite{MMSR:2022}, \texttt{APDA-Inexact}~\cite{SKR:2022}, and \texttt{RandProx}~\cite{CR:2023}.
The optimal communication complexity of \DualFL{} is achieved without relying on randomness; all the statements in the algorithm are deterministic.
This feature is shared with some recent federated learning algorithms such as \texttt{APDA-Inexact}~\cite{SKR:2022} and \texttt{5GCS}~\cite{GMR:2022}.
A distinct novelty of \DualFL{} is its communication acceleration, even when the cost function is either nonsmooth or non-strongly convex.
Among the existing fifth generation of federated learning algorithms, only \texttt{RandProx} has been proven to be convergent in the smooth non-strongly convex regime, with an $O(1/n)$-convergence rate of the energy error~\cite[Theorem~11]{CR:2023}.
However, this rate is the same of those of federated learning algorithms without communication acceleration such as \texttt{Scaffold}~\cite{KKMRSS:2020} and \texttt{FedDyn}~\cite{AZMMWS:2021}.
In contrast, \DualFL{} achieves the $O((1/\sqrt{\epsilon}) \log (1/\epsilon))$-communication complexity with respect to the gradient error, which has not been achieved by the existing algorithms.
Furthermore, not only communication acceleration but also convergence to a solution in the nonsmooth strongly convex regime have not been addressed by the existing fifth generation algorithms.

In the local problem~\eqref{local_primal} of \DualFL{}, we minimize not only the local cost function $f_j (\theta_j)$ but also an additional term $- \nu \langle \zeta_j^{(n)}, \theta_j \rangle$.
That is, $- \nu \zeta_j^{(n)}$ serves as a gradient shift to mitigate client drift and accelerate convergence.
In this viewpoint, \DualFL{} can be classified as a federated learning algorithm with gradient shift.
This class includes other methods such as \texttt{Scaffold}~\cite{KKMRSS:2020}, \texttt{FedDyn}~\cite{AZMMWS:2021}, \texttt{S-Local-SVRG}~\cite{GHR:2021}, \texttt{FedLin}~\cite{MJPH:2021}, and \texttt{Scaffnew}~\cite{MMSR:2022}.
Meanwhile, \DualFL{} belongs to the class of primal-dual methods for federated learning, e.g., \texttt{FedPD}~\cite{ZHDYL:2021}, \texttt{FedDR}~\cite{TPPN:2021}, \texttt{APDA-Inexact}~\cite{SKR:2022}, and \texttt{5GCS}~\cite{GMR:2022}.
While almost of the existing methods utilize a consensus reformulation of~\eqref{FL}~(see~\cite[Equation~(6)]{MMSR:2022}), \DualFL{} is based on a certain dual formulation of~\eqref{FL}, as we will see in \cref{Sec:Math}.
More precisely, we will show that \DualFL{} is obtained by applying predualization~\cite{LG:2019,LN:2017} to an accelerated forward-backward splitting algorithm~\cite{BT:2009,CP:2016,RC:2022} for the dual problem.
The dual problem has a particular structure that makes the forward-backward splitting algorithm equivalent to the prerelaxed nonlinear block Jacobi method~\cite{LP:2019}, which belongs to a broad class of parallel subspace correction methods~\cite{Xu:1992} for convex optimization~\cite{Park:2020,TX:2002}.

\cref{Table:existing} provides an overview of the comparison between \DualFL{} and other fifth-generation federated learning algorithms discussed above.

\section{Mathematical theory}
\label{Sec:Math}
This section provides a mathematical theory for \DualFL{}.
We establish a duality relation between \DualFL{} and an accelerated forward-backward splitting algorithm~\cite{BT:2009,CP:2016,RC:2022} applied to a certain dual formulation of the model problem~\eqref{FL}.
Utilizing this duality relation, we derive the convergence theorems presented in this paper, namely \cref{Thm:nonsmooth,Thm:smooth}.
Moreover, the duality relation provides a rationale for naming the proposed algorithm as \DualFL{}.
Throughout this section, we may assume that each $f_j$ in~\eqref{FL} is $\mu$-strongly convex, as \DualFL{} for a non-strongly convex problem utilizes the strongly convex regularization~\eqref{FL_reg}.

\subsection{Dual formulation}
We introduce a dual formulation of the model federated learning problem~\eqref{FL} that is required for the convergence analysis.
For the sake of completeness, we first present key features of the Fenchel--Rockafellar duality; see also~\cite{CP:2016,Rockafellar:2015}.
For a proper function $F \colon X \rightarrow \overline{\mathbb{R}}$ defined on a Euclidean space $X$, the effective domain $\operatorname{dom} F$ of $F$ is defined by
\begin{equation*}
    \operatorname{dom} F = \left\{ x \in X : F(x) < \infty \right\}.
\end{equation*}
Recall that the Legendre--Fenchel conjugate of $F$ is denoted by $F^* \colon X \rightarrow \overline{\mathbb{R}}$, i.e.,
\begin{equation*}
    F^* (p) = \sup_{x \in X} \left\{ \left< p, x \right> - F(x) \right\}, \quad p \in X.
\end{equation*}
One may refer to~\cite{Rockafellar:2015} for the elementary properties of the Legendre--Fenchel conjugate.
In \cref{Prop:FR}, we summarize the notion of the Fenchel--Rockafellar duality~\cite[Corollary~31.2.1]{Rockafellar:2015}, which plays an important role in the convergence analysis of \DualFL{}.

\begin{proposition}[Fenchel--Rockafellar duality]
\label{Prop:FR}
Let $X$ and $Y$ be Euclidean spaces.
Consider the minimization problem
\begin{equation}
    \label{primal}
    \min_{x \in X} \left\{ F(x) + G(Kx) \right\},
\end{equation}
where $K \colon X \rightarrow Y$ is a linear operator and $F \colon X \rightarrow \overline{\mathbb{R}}$ and $G \colon Y \rightarrow \overline{\mathbb{R}}$ are proper, convex, and lower semicontinuous functions.
If there exists $x_0 \in X$ such that $x_0$ is in the relative interior of $\operatorname{dom} F$ and $K x_0$ is in the relative interior of $\operatorname{dom} G$, then the following relation holds:
\begin{align}
\nonumber
\min_{x \in X} \left\{ F(x) + G(Kx) \right\} 
\label{dual}
= - \min_{y \in Y} \left\{ F^* (-K^{\mathrm{T}} y) + G^* (y) \right\}.
\end{align}
Moreover, the primal solution $x^* \in X$ and the dual solution $y^* \in Y$ satisfy
\begin{equation}
    \label{pd_relation_abstract}
    -K^{\mathrm{T}} y^* \in \partial F(x^*), \quad
    K x^* \in \partial G(y^*).
\end{equation}
\end{proposition}

Leveraging the Fenchel--Rockafellar duality, we are able to derive the dual formulation of the model federated learning problem.
For a positive constant $\nu$, the problem~\eqref{FL} can be rewritten as follows:
\begin{equation}
    \label{FL_modified}
    \min_{\theta \in \Omega} \left\{ \frac{1}{N} \sum_{j=1}^N g_j (\theta) + \frac{\nu}{2} \| \theta \|^2 \right\},
\end{equation}
where $g_j (\theta) = f_j (\theta) - \frac{\nu}{2} \| \theta \|^2$.
By the $\mu$-strong convexity of each $f_j$, $g_j$ is convex if $\nu \in (0, \mu]$.
In~\eqref{primal}, if we set
\begin{equation*}
    X = \Omega, \quad
    Y = \Omega^N, \quad
    K = \begin{bmatrix} I \\ \vdots \\ I \end{bmatrix}, \quad
    F(\theta) = \frac{N\nu}{2} \| \theta \|^2, \quad
    G(\boldsymbol{\xi}) = \sum_{j=1}^N g_j (\xi_j),
\end{equation*}
for $\theta \in \Omega$ and $\boldsymbol{\xi} \in \Omega^N$, then we obtain~\eqref{FL_modified}.
Here, $I$ is the identity matrix on $\Omega$.
By the definition of the Legendre--Fenchel conjugate, we readily get
\begin{equation*}
    F^* (\theta) = \frac{1}{2N\nu} \| \theta \|^2, \quad
    G^* (\boldsymbol{\xi}) = \sum_{j=1}^N g_j^* (\xi_j).
\end{equation*}
Hence, invoking \cref{Prop:FR} yields the dual problem
\begin{equation}
    \label{FL_dual}
    \min_{\boldsymbol{\xi} \in \Omega^N} \left\{ E_{\mathrm{d}} ( \boldsymbol{\xi} ) := \sum_{j=1}^N g_j^* (\xi_j) + \frac{1}{2 N \nu} \left\| \sum_{j=1}^N \xi_j \right\|^2 \right\}.
\end{equation}
We note that problems of the form~\eqref{FL_dual} have been applied in some limited cases in machine learning, such as support vector machines~\cite{HCLKS:2008} and logistic regression~\cite{YHL:2011}.
Very recently, the dual problem~\eqref{FL_dual} was utilized in federated learning in~\cite{GMR:2022}.

Let $\boldsymbol{\xi}^* \in \Omega^N$ denote a solution of~\eqref{FL_dual}.
Invoking~\eqref{pd_relation_abstract}, we obtain the primal-dual relation
\begin{equation}
    \label{pd_relation}
    \theta^* = -\frac{1}{N \nu} \sum_{j=1}^N \xi_j^*, \quad
    \xi_j^* = \nabla g_j (\theta^*).
\end{equation}
between the primal solution $\theta^*$ and the dual solution $\boldsymbol{\xi}^*$.

\subsection{Inexact \texttt{FISTA}}
For $\boldsymbol{\xi} \in \Omega^N$, let
\begin{equation*}
    F_{\mathrm{d}} (\boldsymbol{\xi}) = \frac{1}{2N \nu} \left\| \sum_{j=1}^N \xi_j \right\|^2, \quad
    G_{\mathrm{d}} (\boldsymbol{\xi}) = \sum_{j=1}^N g_j^* (\xi_j).
\end{equation*}
Then the dual problem~\eqref{FL_dual} is rewritten as the following composite optimization problem~\cite{Nesterov:2013}:
\begin{equation}
    \label{composite}
    \min_{\boldsymbol{\xi} \in \Omega^N} \left\{ E_{\mathrm{d}} (\boldsymbol{\xi}) := F_{\mathrm{d}} ( \boldsymbol{\xi}) + G_{\mathrm{d}} ( \boldsymbol{\xi}) \right\}. 
\end{equation}
By the Cauchy--Schwarz inequality, $F_{\mathrm{d}}$ is $\nu^{-1}$-smooth.
Moreover, under $\mu$-strong convexity and $L$-smoothness assumptions on each $f_j$, $G_{\mathrm{d}}$ is $(L - \nu)^{-1}$-strongly convex if $\nu \in (0, \mu]$.
Since~\eqref{composite} is a composite optimization problem, forward-backward splitting algorithms are well-suited to solve it.
Among several variants of forward-backward splitting algorithms, we focus on an inexact version of \texttt{FISTA}~\cite{BT:2009} proposed in~\cite{RC:2022}, which accommodates strongly convex objectives and inexact proximal operations.
Inexact \texttt{FISTA} with the fixed step size $\nu$ applied to~\eqref{composite} is summarized in \cref{Alg:FISTA}, in the form suitable for our purposes.

\begin{algorithm}
\caption{Inexact \texttt{FISTA} for the dual problem~\eqref{composite}}
\label{Alg:FISTA}
\begin{algorithmic}[H]
\STATE Given $\rho \geq 0$, $\nu > 0$, and $\{ \delta_n \}_{n = 0}^{\infty}$,
\STATE set $\boldsymbol{\xi}^{(0)} = \boldsymbol{\eta}^{(0)} = \textbf{0} \in \Omega^N$, and $t_0 = 1$.
\FOR{$n= 0,1,2, \dots$}
\STATE
\begin{equation}
\label{prox}
\displaystyle
\boldsymbol{\xi}^{(n+1)} \approx \operatornamewithlimits{\arg\min}_{\boldsymbol{\xi} \in \Omega^N} \left\{ E_{\mathrm{d}}^{n} (\boldsymbol{\xi}) := \langle \nabla F_{\mathrm{d}} (\boldsymbol{\eta}^{(n)}), \boldsymbol{\xi} - \boldsymbol{\eta}^{(n)} \rangle + \frac{1}{2\nu} \| \boldsymbol{\xi} - \boldsymbol{\eta}^{(n)} \|^2 + G_{\mathrm{d}} (\boldsymbol{\xi}) \right\}
\end{equation}
such that $E_{\mathrm{d}}^{n} (\boldsymbol{\xi}^{(n+1)}) - \min E_{\mathrm{d}}^{n} \leq \delta_n$.
\STATE
\begin{equation}
    \label{overrelaxation}
    \boldsymbol{\eta}^{(n+1)} = (1 + \beta_n) \boldsymbol{\xi}^{(n+1)} - \beta_n \boldsymbol{\xi}^{(n)},
\end{equation}
where $\beta_n$ is given by~\eqref{FISTA}.
\ENDFOR
\end{algorithmic}
\end{algorithm}

We state the convergence theorems of \cref{Alg:FISTA}, which are essential ingredients for the convergence analysis of \DualFL{}.
Recall that, if each $f_j$ in~\eqref{FL} is $\mu$-strongly convex and $\nu \in (0, \mu]$, then
$F_{\mathrm{d}}$ in~\eqref{composite} is $\nu^{-1}$-smooth.
Hence, we have the following convergence theorem of \cref{Alg:FISTA}~\cite[Corollary~3.3]{RC:2022}.

\begin{proposition}
\label{Prop:FISTA_nonsmooth}
Suppose that each $f_j$, $1 \leq j \leq N$, in~\eqref{FL} is $\mu$-strongly convex for some $\mu > 0$.
In addition, suppose that the error sequence $\{ \delta_n \}$ in \cref{Alg:FISTA} is given by
\begin{equation*}
    \delta_n = \frac{b_n}{(n+1)^2}, \quad n \geq 0,
\end{equation*}
where $\{ b_n \}$ satisfies $\sum_{n=0}^{\infty} \sqrt{b_n} < \infty$.
If we choose the hyperparameters $\rho$ and $\nu$ in \cref{Alg:FISTA} such that $\rho = 0$ and $\nu \in (0, \mu]$, then we have
\begin{equation*}
    E_{\mathrm{d}} (\boldsymbol{\xi}^{(n)}) - E_{\mathrm{d}} (\boldsymbol{\xi}^*) \lesssim \frac{1}{ n^2}, \quad n \geq 0.
\end{equation*}
\end{proposition}

If we further assume that each $f_j$ in~\eqref{FL} is $L$-smooth, then $G_{\mathrm{d}}$ in~\eqref{composite} is $(L-\nu)^{-1}$-strongly convex.
In this case, we have the following improved convergence theorem for \cref{Alg:FISTA}~\cite[Corollary~3.4]{RC:2022}.

\begin{proposition}
\label{Prop:FISTA_smooth}
Suppose that each $f_j$, $1 \leq j \leq N$, in~\eqref{FL} is $\mu$-strongly convex and $L$-smooth for some $\mu, L > 0$.
In addition, suppose that the error sequence $\{ \delta_n \}$ in \cref{Alg:FISTA} is given by
\begin{equation*}
    \delta_n = a^n, \quad n \geq 0,
\end{equation*}
where $a \in [0, 1 - \sqrt{\rho})$.
If we choose the hyperparameters $\rho$ and $\nu$ in \cref{Alg:FISTA} such that $\rho \in (0, \nu / L]$ and $\nu \in (0, \mu]$, then we have
\begin{equation*}
    E_{\mathrm{d}} (\boldsymbol{\xi}^{(n)}) - E_{\mathrm{d}} (\boldsymbol{\xi}^*) \lesssim  (1 - \sqrt{\rho})^n, \quad n \geq 0.
\end{equation*}
\end{proposition}

The dual problem~\eqref{FL_dual} has a particular structure that allows \cref{Alg:FISTA} to be viewed as a parallel subspace correction method for \eqref{FL_dual}~\cite{Park:2020,TX:2002,Xu:1992}.
That is, the proximal problem~\eqref{prox} can be decomposed into $N$ independent subproblems, each defined in terms of $\xi_j$ for $1 \leq j \leq N$.
Specifically, \cref{Lem:Jacobi} shows that \cref{Alg:FISTA} is equivalent to the prerelaxed block Jacobi method, which was introduced in~\cite{LP:2019}.

\begin{lemma}
\label{Lem:Jacobi}
In \cref{Alg:FISTA}, suppose that $\tilde{\boldsymbol{\xi}}^{(n+1)} \in \Omega^N$ satisfies
\begin{equation*}
    \tilde{\xi}_j^{(n+1)} \approx \operatornamewithlimits{\arg\min}_{\xi_j \in \Omega} \left\{ \tilde{E}_d^{n,j} (\xi_j) := g_j^* (\xi_j) + \frac{1}{2\nu} \left\| \xi_j - \eta_j^{(n)} + \frac{1}{N} \sum_{i=1}^N \eta_i^{(n)} \right\|^2 \right\}
\end{equation*}
such that $\tilde{E}_d^{n,j} (\tilde{\xi}_j^{(n+1)}) - \min \tilde{E}_d^{n,j} \leq \delta_n / N$ for $1 \leq j \leq N$.
Then $\tilde{\boldsymbol{\xi}}^{(n+1)}$ solves the proximal problem~\eqref{prox} such that $E_{\mathrm{d}}^{n} (\tilde{\boldsymbol{\xi}}^{(n+1)}) - \min E_{\mathrm{d}}^{n} \leq \delta_n$.
\end{lemma}
\begin{proof}[Proof of \cref{Lem:Jacobi}]
By direct calculation, we get
\begin{equation*}
    \sum_{j=1}^N \tilde{E}_{\mathrm{d}}^{n,j} ( \xi_j) 
    = E_{\mathrm{d}}^n (\boldsymbol{\xi}) + \text{constant}
\end{equation*}
for any $\boldsymbol{\xi} \in \Omega^N$, which completes the proof.
\end{proof}

\subsection{Duality between \texorpdfstring{\texttt{DualFL} and \texttt{FISTA}}{DualFL and FISTA}}
An important observation is that there exists a duality relation between the sequences generated by \cref{Alg:FISTA} and those generated by \DualFL{}.
In \DualFL{}, we define two auxiliary sequences $\{ \boldsymbol{\xi}^{(n)} \}$ and $\{ \boldsymbol{\eta}^{(n)} \}$ as follows:
\begin{subequations}
\label{auxiliary}
\begin{align}
    \label{xi}
    \xi_j^{(n+1)} &= \nu ( \zeta_j^{(n)} - \theta_j^{(n+1)}), 
    \quad \xi_j^{(0)} = 0, \\
    \label{eta}
    \eta_j^{(n+1)} &= \nu (\zeta_j^{(n+1)} - (1+ \beta_n) \theta^{(n+1)} + \beta_n \theta^{(n)}),
    \quad \eta_j^{(0)} = 0,
\end{align}
\end{subequations}
for $n \geq 0$ and $1 \leq j \leq N$.
\cref{Lem:duality} summarizes the duality relation between \DualFL{} and \cref{Alg:FISTA}; the sequences $\{ \boldsymbol{\xi}^{(n)} \}$ and $\{ \boldsymbol{\eta}^{(n)} \}$ defined in~\eqref{auxiliary} agree with those generated by \cref{Alg:FISTA}.

\begin{lemma}
\label{Lem:duality}
Suppose that each $f_j$, $1 \leq j \leq N$, in~\eqref{FL} is $\mu$-strongly convex for some $\mu > 0$.
In addition, suppose that the number of local iterations for the $j$th client at the $n$th epoch of \DualFL{} is large enough to satisfy
\begin{equation*}
    \Gamma^{n,j} (\theta_j^{(n+1)}) \leq \frac{\delta_n}{N}
\end{equation*}
for some $\delta_n > 0$~($1 \leq j \leq N$, $n \geq 0$).
Then the sequences $\{ \boldsymbol{\xi}^{(n)} \}$ and $\{ \boldsymbol{\eta}^{(n)} \}$ defined in~\eqref{auxiliary} agree with those generated by \cref{Alg:FISTA} for the dual problem~\eqref{composite}.
\end{lemma}
\begin{proof}
It suffices to show that the sequences $\{ \boldsymbol{\xi}^{(n)} \}$ and $\{ \boldsymbol{\eta}^{(n)} \}$ defined in~\eqref{auxiliary} satisfy~\eqref{prox} and~\eqref{overrelaxation}.
We first observe that
\begin{equation}
    \label{zeta_sum}
    \sum_{i=1}^N \zeta_i^{(n)} = 0, \quad n \geq 0,
\end{equation}
which can be easily derived by mathematical induction with~\eqref{theta} and~\eqref{zeta}.
Now, we take any $n \geq 0$ and $1 \leq j \leq N$.
By direct calculation, we obtain
\begin{equation*} \begin{split}
    \sum_{i=1}^N \eta_i^{(n)}
    &\stackrel{\eqref{eta}}{=} \nu \sum_{i=1}^N \zeta_i^{(n)} - \nu N (1+ \beta_n) \theta^{(n)} + \nu N \beta_n \theta^{(n-1)} \\
    &\stackrel{\eqref{zeta_sum}}{=} - N \nu ( 1 + \beta_n ) \theta^{(n)} + N \nu \beta_n \theta^{(n-1)} \\
    &\stackrel{\eqref{eta}}{=} N \eta_j^{(n)} - N \nu \zeta^{(n)}.
\end{split} \end{equation*}
Hence, we get
\begin{equation}
    \label{zeta_eta}
    \nu \zeta_j^{(n)} = \eta_j^{(n)} - \frac{1}{N} \sum_{i=1}^N \eta_i^{(n)}.
\end{equation}
Combining~\eqref{local_dual},~\eqref{zeta_eta}, and \cref{Lem:Jacobi} implies that $\{ \boldsymbol{\xi}^{(n)} \}$ and $\{ \boldsymbol{\eta}^{(n)} \}$ satisfy~\eqref{prox}.
On the other hand, we obtain by direct calculation that
\begin{equation*} \begin{split}
    (1+\beta_n) \xi_j^{(n+1)} - \beta_n \xi_j^{(n)}
    &\stackrel{\eqref{xi}}{=} \nu (1+ \beta_n) (\zeta_j^{(n)} - \theta_j^{(n+1)}) - \nu \beta_n (\zeta_j^{(n-1)} - \theta_j^{(n)}) \\
    &\stackrel{\eqref{zeta}}{=} \nu \zeta_j^{(n+1)} - \nu ( 1+ \beta_n) \theta^{(n+1)} - \nu \beta_n \theta^{(n)} \\
    &\stackrel{\eqref{eta}}{=} \eta_j^{(n+1)},
\end{split} \end{equation*}
which implies that $\{ \boldsymbol{\xi}^{(n)} \}$ and $\{ \boldsymbol{\eta}^{(n)} \}$ satisfy~\eqref{overrelaxation}.
This completes the proof.
\end{proof}

\Cref{Lem:duality} implies that \DualFL{} is a predualization~\cite{LG:2019,LN:2017} of \cref{Alg:FISTA}.
Namely, \DualFL{} can be constructed by transforming the dual sequence $\{ \boldsymbol{\xi}^{(n)} \}$ generated by \cref{Alg:FISTA} into the primal sequence $\{ \theta^{(n)} \}$ by leveraging the primal-dual relation~\eqref{pd_relation}.

\subsection{Proofs of \texorpdfstring{\cref{Thm:nonsmooth,Thm:smooth}}{Theorems 3.2 and 3.3}}
Finally, the main convergence theorems for \DualFL{}, \cref{Thm:nonsmooth,Thm:smooth}, can be derived by combining the optimal convergence properties of \cref{Alg:FISTA} presented in \cref{Prop:FISTA_nonsmooth,Prop:FISTA_smooth} and the duality relation presented in \cref{Lem:duality}.

\begin{proof}[Proof of \cref{Thm:nonsmooth,Thm:smooth}]
Thanks to \cref{Lem:duality}, the sequence $\{ \boldsymbol{\xi}^{(n)} \}$ defined in~\eqref{xi} satisfies the convergence properties given in \cref{Prop:FISTA_nonsmooth,Prop:FISTA_smooth}, i.e.,
\begin{equation}
    \label{dual_estimate}
    E_{\mathrm{d}} ( \boldsymbol{\xi}^{(n)} ) - E_{\mathrm{d}} \left( \boldsymbol{\xi}^* \right)  \lesssim
    \begin{cases}
    \displaystyle \frac{1}{n^2}, & \text{ in the case of \cref{Thm:nonsmooth},} \\
    \displaystyle \left( 1 - \sqrt{\rho} \right)^{n}, & \text{ in the case of \cref{Thm:smooth}.}
    \end{cases}
\end{equation}
Next, we derive an estimate for the primal norm error $\| \theta^{(n)} - \theta^* \|$ by a similar argument as in  of~\cite[Corollary~1]{LP:2021}.
Note that the dual cost function $E_{\mathrm{d}}$ given in~\eqref{FL_dual} is $\frac{1}{N \nu}$-strongly convex relative to a seminorm $| \boldsymbol{\xi} | = \| \sum_{j=1}^N \xi_j \|$.
Hence, by~\eqref{pd_relation},~\eqref{xi}, and~\eqref{zeta_sum}, we obtain
\begin{equation}
    \label{solution_error}
    \| \theta^{(n)} - \theta^* \|^2
    = \frac{1}{N^2 \nu^2} \left\| \sum_{j=1}^N \left( \xi_j^{(n)} - \xi_j^* \right) \right\|^2 
    \leq \frac{2}{N \nu} \left( E_{\mathrm{d}} (\boldsymbol{\xi}^{(n)})  - E_{\mathrm{d}} (\boldsymbol{\xi}^*) \right).
\end{equation}
Meanwhile, it is clear that
\begin{equation}
    \label{energy_error}
    E(\theta^{(n)}) - E(\theta^*) \leq \frac{L}{2} \| \theta^{(n)} - \theta^* \|^2
\end{equation}
under the $L$-smoothness assumption.
Combining~\eqref{dual_estimate},~\eqref{solution_error}, and~\eqref{energy_error} completes the proof.
\end{proof}

\section{Numerical experiments}
\label{Sec:Numerical}
In this section, we present numerical results that demonstrate the performance of \DualFL{}.
All the algorithms were programmed using MATLAB~R2022b and performed on a desktop equipped with AMD Ryzen~5 5600X CPU~(3.7GHz, 6C), 40GB RAM, NVIDIA GeForce GTX~1660 SUPER GPU with 6GB GDDR6 memory, and the operating system Windows~10 Pro.

\begin{table}
    \centering
    \caption{Description of the hyperparameters appearing in the benchmark algorithms \texttt{FedPD}, \texttt{FedDR}, \texttt{FedDualAvg}, \texttt{FedDyn}, and \texttt{Scaffnew}, \texttt{APDA-Inexact}, \texttt{5GCS}, and the proposed \texttt{DualFL}{}.
    We use the notation for each hyperparameter as given in the original paper.
    The value of each hyperparameter is determined using a grid search.}
    \resizebox{\textwidth}{!}{ 
    \begin{tabular}{m{0.22\textwidth}m{0.1\textwidth}m{0.18\textwidth}m{0.3\textwidth}m{0.1\textwidth}m{0.1\textwidth}}
    \multirow{2}{*}{Algorithm} & Hyper- & \multirow{2}{*}{Description} & \multirow{2}{*}{Grid} & \multicolumn{2}{c}{Value} \\
    \cline{5-6}
    & param. & & & MNIST & CIFAR-10\\
    \hline\hline
    \texttt{FedPD}~\cite{ZHDYL:2021}
    & $\eta$ & Local penalty param. & $\{ 10^{-m}: m \in \mathbb{Z}_{\geq 0} \}$ & \multicolumn{2}{c}{$10^{-4}$} \\
    \hline
    \multirow{2}{*}{\texttt{FedDR}~\cite{TPPN:2021}}
    & $\eta$ & Local penalty param. & $\{ 10^{-m}: m \in \mathbb{Z}_{\geq 0} \}$ & \multicolumn{2}{c}{$10^{-4}$} \\
    & $\alpha$ & Overrelax.\ param. & $\{ 1, 2 \}$ & \multicolumn{2}{c}{$1$} \\
    \hline
    \multirow{2}{*}{\texttt{FedDualAvg}~\cite{YZR:2021}}
    & $\eta_c$ & Client learning rate & $\{ 10^{-m}: m \in \mathbb{Z}_{\geq 0} \}$ & $10^{-3}$ & $10^{-4}$\\
    & $\eta_s$ & Server learning rate & $\{ 1\}$ & \multicolumn{2}{c}{$1$} \\
    & $K$ & \# local gradient steps & $\{ 10^m : m \in \mathbb{Z}_{\geq 0} \}$ & $10$ & $10^2$ \\
    \hline
    \texttt{FedDyn~\cite{AZMMWS:2021}}    & $\alpha$ & Regularization param. & $\{ 10^{m}: m \in \mathbb{Z}_{\geq 0} \}$ & \multicolumn{2}{c}{$10^3$} \\
    \hline
    \multirow{2}{*}{\texttt{Scaffnew}~\cite{MMSR:2022}}     & $\gamma$ & Learning rate & $\{ 10^{-m}: m \in \mathbb{Z}_{\geq 0} \}$ & $10^{-5}$ & $10^{-7}$ \\
    & $p$ & Commun.\ probability & $\{10^{-m} : m \in \mathbb{Z}_{\geq 0} \}$ & $10^{-1}$ & $10^{-3}$ \\
    \hline
    \multirow{4}{*}{\texttt{APDA-Inexact}~\cite{SKR:2022}}     & $\eta_x$ & Primal learning rate & $\{ 10^{-m}: m \in \mathbb{Z}_{\geq 0} \}$ & $10^{-3}$ & $10^{-4} $\\
    & $\eta_y$ & Dual learning rate I& $\{1/32 \eta_x \}$ & \multicolumn{2}{c}{$1/32 \eta_x$} \\
    & $\beta_y$ & Dual learning rate II & $\{ 10^{-m}: m \in \mathbb{Z}_{\geq 0} \}$ & $10^{-2}$ & $10^{-3}$ \\
    & $\theta$ & Overrelax.\ param. & $\{\ \max \{ \frac{2}{2+\mu\eta_x}, 1-\beta_y \eta_y \} \}$ & \multicolumn{2}{c}{$\max \{ \frac{2}{2+\mu\eta_x}, 1-\beta_y \eta_y \}$} \\
    \hline
    \multirow{2}{*}{\texttt{5GCS}~\cite{GMR:2022}}     & $\gamma$ & Primal learning rate & $\{ 10^{-m}: m \in \mathbb{Z}_{\geq 0} \}$ & \multicolumn{2}{c}{$10^{-4}$} \\
    & $\tau$ & Dual learning rate & $\{1/2\gamma N \}$ & \multicolumn{2}{c}{$1/2\gamma N$} \\
    \hline
    \multirow{2}{*}{\texttt{DualFL}{}} & $\rho$ & Momentum param. & $\{ m\times 10^{-3} : m \in \mathbb{Z}_{\geq 0} \}$ & $3 \times 10^{-3}$ & $0$\\
    & $\nu$ & Param.\ for duality & $\{ \mu \}$ & \multicolumn{2}{c}{$\mu$} \\
    \end{tabular}
    }
    \label{Table:hyperparameters}
\end{table}

As benchmarks, we choose the following recent federated learning algorithms: \texttt{FedPD}~\cite{ZHDYL:2021}, \texttt{FedDR}~\cite{TPPN:2021}, \texttt{FedDualAvg}~\cite{YZR:2021}, \texttt{FedDyn}~\cite{AZMMWS:2021}, \texttt{Scaffnew}~\cite{MMSR:2022}, \texttt{APDA-Inexact}~\cite{SKR:2022}, and \texttt{5GCS}~\cite{GMR:2022}.
All the hyperparameters appearing in these algorithms and \texttt{DualFL} are tuned by a grid search; see \cref{Table:hyperparameters} for details of the tuned hyperparameters.

\begin{remark}
\label{Rem:primal}
While we also conducted experiments with several primal federated learning algorithms such as \texttt{FedAvg}~\cite{MMRHA:2017}, \texttt{FedCM}~\cite{XWWY:2021}, and \texttt{FedSAM}~\cite{QLDLTL:2022}, which do not rely on duality in their mechanisms, we do not present their results as their performances were not competitive compared to other methods.
\end{remark}

To solve the local problems encountered in these algorithms, we employ the optimized gradient method with adaptive restart~(\texttt{AOGM}) proposed in~\cite{KF:2018}, with the stop criterion in which the algorithm terminates when the relative energy difference becomes less than $10^{-12}$.
In each iteration of \texttt{AOGM}, the step size is determined using the full backtracking scheme introduced in~\cite{CC:2019}.

To test the performance of the algorithms, we use multinomial logistic regression problem, which is stated as
\begin{equation}
    \label{LR}
    \min_{\theta = (w, b) \in \mathbb{R}^{d \times k} \times \mathbb{R}^k}
    \left\{ \frac{1}{n} \sum_{j=1}^n \log \left( \sum_{l=1}^k e^{  (w_l \cdot x_j + b_l) - (w_{y_j} \cdot x_j + b_{y_j} ) } \right) + \frac{\mu}{2} \| \theta \|^2 \right\},
\end{equation}
where $\{ (x_j , y_j ) \}_{j=1}^n \subset \mathbb{R}^d \times \{ 1, \dots, k \}$ is a labeled dataset.
In~\eqref{LR}, we set the regularization parameter $\mu = 10^{-2}$.
We use the MNIST~\cite{LBBH:1998}~($k = 10$, $n = 60,000$, $d = 28 \times 28$) and CIFAR-10~\cite{Krizhevsky:2009}~($k = 10$, $n = 60,000$, $d = 32 \times 32 \times 3$) training datasets.
We assume that the dataset is evenly distributed to $N$ clients to form $f_1$, \dots, $f_N$, so that~\eqref{LR} is expressed in the form~\eqref{FL}.
A reference solution $\theta^* \in \mathbb{R}^{(d+1) \times k}$ of~\eqref{LR} is obtained by a sufficient number of damped Newton iterations~\cite{BV:2004}.

\begin{figure}
  \centering
  \includegraphics[width=\textwidth]{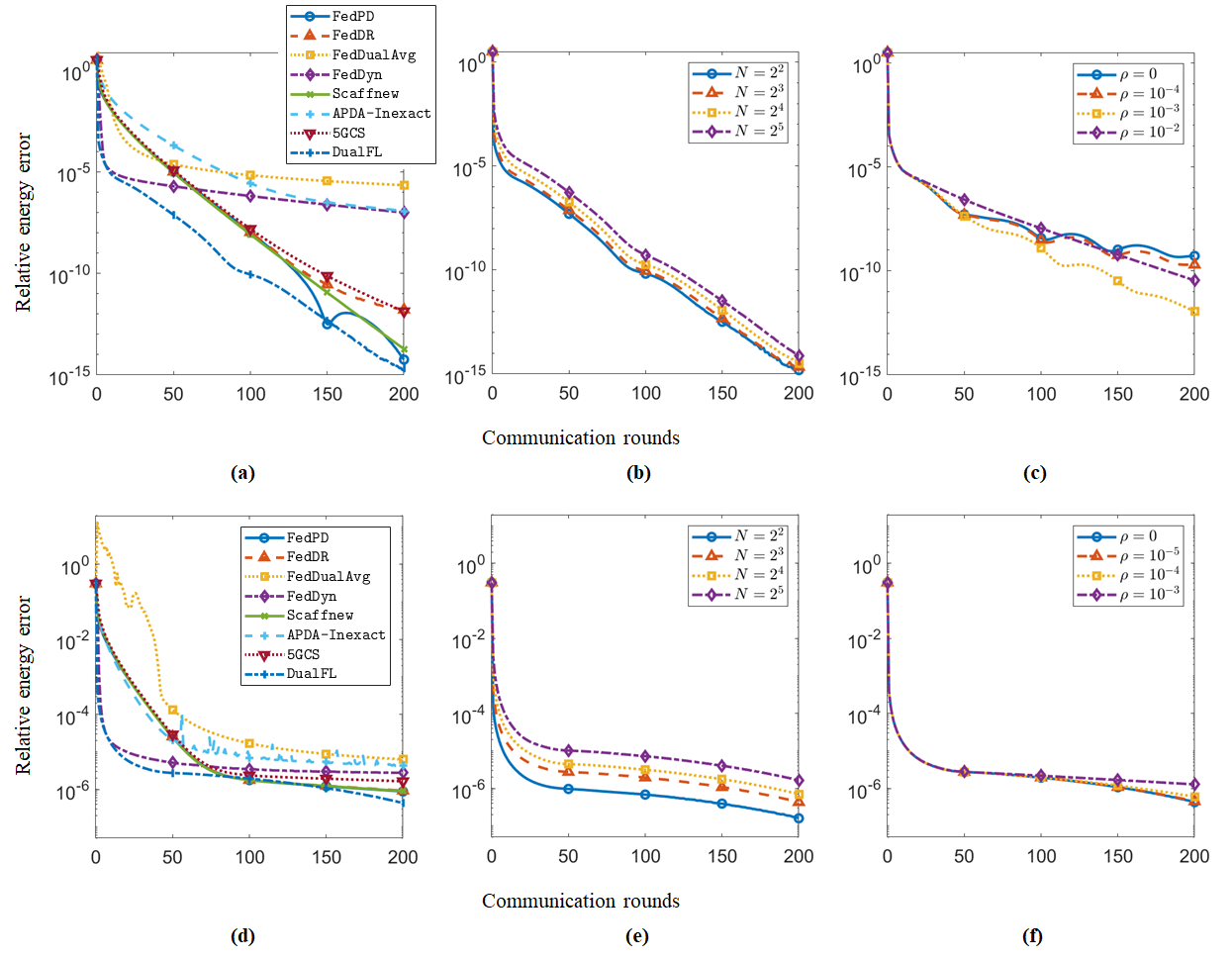}
  \caption{Relative energy error $\frac{E(\theta) - E(\theta^*)}{E(\theta^*)}$ with respect to the number of communication rounds in various training algorithms for multinomial logistic regression on the {\rm{\textbf{(a--c)}}}~MNIST and {\rm{\textbf{(d--f)}}}~CIFAR-10 training dataset.
  {\rm{\textbf{(a, d)}}}~Comparison of \texttt{DualFL}{} with benchmark algorithms.
  {\rm{\textbf{(b, e)}}}~Convergence of \texttt{DualFL}{} when the number of clients $N$ changes.
  {\rm{\textbf{(c, f)}}}~Convergence of \texttt{DualFL}{} when the value of the hyperparameter $\rho$ changes.
  }
  \label{Fig:numerical}
\end{figure}


Numerical results are presented in \cref{Fig:numerical}.
\cref{Fig:numerical}(a, d) displays the convergence behavior of the benchmark algorithms, along with \DualFL{}, when $N = 2^3$.
While the linear convergence rate of \DualFL{} appears to be similar to those of \texttt{FedPD}, \texttt{FedDR}, and \texttt{Scaffnew}, the energy curve of \DualFL{} is consistently lower than those of the other algorithms because \DualFL{} achieves faster energy decay in the first several iterations, similar to \texttt{FedDyn}.
That is, the \DualFL{} loss decays as fast as \texttt{FedDyn} in the first several iterations, and then the linear decay rate of \DualFL{} becomes similar to those of \texttt{FedPD}, \texttt{FedDR}, and \texttt{Scaffnew}.
\cref{Fig:numerical}(b, e) verifies that the convergence rate of \DualFL{} does not deteriorate even if the number of clients $N$ becomes large.
That is, \DualFL{} is robust to a large number of clients.
Finally, \cref{Fig:numerical}(c, f) illustrates the convergence behavior of \DualFL{} under the condition where $\nu$ is fixed by $\mu$, and the value of $\rho$  are varied.
It can be seen that even when $\rho$ is chosen far from its tuned value, the convergence rate of \DualFL{} does not deteriorate significantly.
This verifies the robustness of \DualFL{} with respect to hyperparameter tuning.

\section{Conclusion}
\label{Sec:Conclusion}
In this paper, we proposed a new federated learning algorithm, called \DualFL{}, based on the duality between the model federated learning problem and a composite optimization problem.
We demonstrated that \DualFL{} achieves communication acceleration, even in cases where the cost function lacks smoothness or strong convexity.
This is the first result in the field of federated learning regarding communication acceleration of general convex cost functions.
Through numerical experiments, we further confirmed that \DualFL{} outperforms several recent federated learning algorithms.

\subsection{Limitations and future works}
A major limitation of this paper is that all the results are based on the convex setting.
Although this limitation is also present in many recent works on federated learning algorithms~\cite{CR:2023,MMSR:2022,SKR:2022}, the nonconvex setting should be considered in future research to cover a wider range of practical machine learning tasks.

While our primary emphasis in this paper lies in enhancing the communication efficiency of training algorithms, we recognize that there are other critical aspects of federated learning, particularly concerning the stochastic nature of machine learning problems~\cite{WBSS:2021}.
Stochasticity can manifest in various aspects of federated learning, including stochastic local training algorithms, client sampling, and communication compression~\cite{GMR:2022,GMR:2023}.
We have successfully tackled the aspect of stochastic local training algorithms in our paper, as \DualFL{} has the flexibility to adopt various local solvers.
However, challenges remain in addressing client sampling and communication compression.
We anticipate that our results can be extended to incorporate client sampling by building upon existing works, such as~\cite{LX:2015,RT:2016}, which focus on accelerated randomized block coordinate descent methods.
Additionally, considering that communication compression can be modeled using stochastic gradients~\cite{GHR:2021}, we consider extending our results by designing an appropriate acceleration scheme for stochastic gradients as a future work.

\bibliographystyle{siamplain}
\bibliography{refs_FL}
\end{document}